%% file: aaai26.tex
\documentclass[letterpaper]{article} 
\usepackage{aaai2026}  
\usepackage{times}  
\usepackage{helvet}  
\usepackage{courier}  
\usepackage[hyphens]{url}  
\usepackage{graphicx} 
\usepackage{natbib}  
\usepackage{caption} 
\usepackage{algorithm}
\usepackage[noend]{algpseudocode}
\usepackage{adjustbox}
\usepackage{newfloat}
\usepackage{listings}
\DeclareCaptionStyle{ruled}{labelfont=normalfont,labelsep=colon,strut=off} 
\floatstyle{ruled}
\newfloat{listing}{tb}{lst}{}
\floatname{listing}{Listing}
\usepackage{amsmath,amsthm}
\usepackage{mathtools}
\usepackage{tikz}
\usetikzlibrary{fit,shapes.geometric, arrows, positioning}
\usepackage{booktabs}
\usepackage{cleveref}
\newtheorem{example}{Example}
\newtheorem{definition}{Definition}
\newtheorem{lemma}{Lemma}

\newif\ifarxiv
\arxivtrue

\input{auxiliary}
\input{bb_common-citations}

\title{Using Certifying Constraint Solvers for Generating Step-wise Explanations \ifarxiv(Extended version)\fi}
\author{
    Ignace Bleukx\textsuperscript{\rm 1},
    Maarten Flippo\textsuperscript{\rm 2},
    Bart Bogaerts\textsuperscript{\rm 1,3},
    Emir Demirovi\'{c}\textsuperscript{\rm 2},
    Tias Guns\textsuperscript{\rm 1}
}
\affiliations{
    \textsuperscript{\rm 1}KU Leuven, Department of Computer Science; Leuven.AI, Celestijnenlaan 200A, 3000 Leuven, Belgium\\
    \textsuperscript{\rm 2}Delft University of Technology, Mekelweg 5, 2628 CD Delft, The Netherlands\\
    \textsuperscript{\rm 3} Vrije Universiteit Brussel, Dept.\ of Computer Science, Pleinlaan 2, Brussels, Belgium\\
    ignace.bleukx@kuleuven.be,
    m.l.flippo@tudelft.nl,
    bart.bogaerts@kuleuven.be,
    e.demirovic@tudelft.nl,
    tias.guns@kuleuven.be
}

\begin{document}

\maketitle

\begin{abstract} 
In the field of Explainable Constraint Solving, it is common to explain to a user why a problem is unsatisfiable.
A recently proposed method for this is to compute a sequence of explanation steps.
Such a step-wise explanation shows individual reasoning steps involving constraints from the original specification, that in the end explain a conflict.
However, computing a step-wise explanation is computationally expensive, limiting the scope of problems for which it can be used.
We investigate how we can use proofs generated by a constraint solver as a starting point for computing step-wise explanations, instead of computing them step-by-step.
More specifically, we define a framework of abstract proofs, in which \textit{both} proofs and step-wise explanations can be represented.
We then propose several methods for converting a proof to a step-wise explanation sequence, with special attention to trimming and simplification techniques to keep the sequence and its individual steps small.
Our results show our method significantly speeds up the generation of step-wise explanation sequences, while the resulting step-wise explanation has a quality similar to the current state-of-the-art.    
\end{abstract}

\ifarxiv
\else
\begin{links}
    \link{Code}{github.com/ML-KULeuven/Proof2Seq}
    \link{Extended version}{/axiv.org/xxxxxx}
\end{links}
\fi

\section{Introduction}

The ultimate goal of \textbf{declarative problem solving} is that users should be able to specify a problem declaratively (i.e., they only need to specify \emph{what} their problem is, not \emph{how} to solve it), and can then use a generic \emph{solver} to find solutions to their problem. 
Over the last decades, solvers \ignore{in various domains, including Constraint Programming (CP),} have become increasingly powerful, allowing them to solve complex, real-world problems with thousands of variables and constraints.
However, due to this complexity, \emph{explaining to a human} why a solver produced a particular answer is still a challenge.

\paragraph{Explanations}
Many types of user-oriented explanations have been developed for constraint programming~\cite{GGO21ExplanationConstraintSatisfactionSurvey}. 
We focus on explanations for problems that do not admit a solution.
More specifically, we focus on methods that explain ``why'' the problem is unsatisfiable.
A popular method is to extract a Minimal Unsatisfiable Subset (MUS) of the input constraints, which allows the user to focus on a problematic subset of the problem specification~\cite{M10MinimalUnsatisfiabilityModelsAlgorithmsApplicationsInvited}.
One issue, however, is that such an MUS can still be very large and hence hard to understand for a user.
To better explain the interaction between constraints, and building on old ideas~\cite{SF96Inference-BasedConstraintSatisfactionSupportsExplanation},
a framework of \textbf{step-wise explanations} was introduced by \citet{BGG21frameworkstep-wiseexplaininghowsolveconstraint}.
In this context, an explanation is a sequence of simple steps, where each step is a simple derivation based on some of the constraints specified by the user.
While the framework was originally conceived to explain satisfiable (puzzle) problems, it has also been used later to step-wise explain \emph{unsatisfiability} and \emph{optimality} \cite{BDGBG23SimplifyingStep-WiseExplanationSequences}. 
A major challenge in the computation of step-wise explanations is how to efficiently find a \emph{short} sequence of \emph{small} explanation steps.
Despite novel algorithms designed exactly for this purpose \cite{GBG23EfficientlyExplainingCSPsUnsatisfiableSubsetOptimization}, \textbf{generating a step-wise explanation remains computationally very expensive}.
Indeed, current approaches rely on multiple calls to an NP-oracle for computing even a single explanation step, and many more so for computing the entire sequence.
The high runtime of existing approaches limits the general application of step-wise explanation methods.
Our goal is to reduce the computational effort for computing a step-wise explanation sequence, by starting from a solver-generated proof instead of constructing it step-by-step.

\paragraph{Proof Logging} 
While combinatorial optimization tools are often used as reliable components in larger systems, their results may not always trustworthy. 
Indeed, there have been numerous reports of solvers reporting faulty answers \cite[e.g.,][]{BB09FuzzingDelta-DebuggingSMTSolvers}. 
The question that naturally arises is: How can we be sure that the answer of a solver is correct? More specifically, if a solver claims that a problem is unsatisfiable or that a solution is optimal, how can we know this is indeed the case? 
One way to provide such a guarantee is to develop \textbf{certifying algorithms} \cite{MMNS11Certifyingalgorithms}, which is also known as  \textbf{proof logging} in the context of combinatorial optimization.
The core idea is that solvers provide a \emph{proof} of correctness of their answer.
This approach has been applied very successfully in SAT, with DRAT being the dominant proof logging format \cite{WHH14DRAT-trimEfficientCheckingTrimmingUsingExpressive}.
Recently, proof logging has also found its way to other combinatorial solving formalisms, including MaxSAT~\cite{VDB22QMaxSATpbCertifiedMaxSATSolver,BBNOV23CertifiedCore-GuidedMaxSATSolving,BBNOPV24CertifyingWithoutLossGeneralityReasoningSolution-Improving,VCB26CertifiedBranch-and-BoundMaxSATSolving}, pseudo-Boolean optimization \cite{KLMNOTV25PracticallyFeasibleProofLoggingPseudo-BooleanOptimization,IVSBBJ26EfficientReliableHitting-SetComputationsImplicitHitting} and CP~\cite{FSIJE24MultiStageProofLoggingFramework,GMN22AuditableConstraintProgrammingSolver}.

\paragraph{Goal: Computing Step-wise Explanations from Proofs} 
There are several similarities, but also some crucial differences between \emph{proofs} generated by certifying solvers and a \emph{step-wise explanation sequence}. 
These differences arise because proofs are designed to be \textbf{machine-verifiable}, while step-wise explanations are designed to be \textbf{human-interpretable}.
Both concepts explain the explanandum through a chain of relatively \textit{simple} steps. 
However, in proofs, these small steps derive new, potentially complex constraints (e.g., a long clause) that can be reused as part of later proof steps.
In contrast, in a step-wise explanation, each step derives a simple \textit{fact}, such as a value assignment of a decision variable by only combining \emph{a few} constraints specified by the user and previously derived facts~\cite{FDGBG26PreferenceElicitationStep-WiseExplanationsLogicPuzzles}.
Hence, solver-generated proofs do not directly provide a user-oriented step-wise explanation.

Another important difference is \emph{how} they are generated. 
Since proofs are generated efficiently during the solving process, a natural question is whether one can \textbf{reuse such a proof to generate a step-wise explanation sequence}. 
In this paper, we do exactly that and contribute the following:

\begin{itemize}
    \item We propose a general framework for unsatisfiability proofs that captures both \emph{proofs} and \emph{explanations};
    \item we propose a suite of techniques which allow to efficiently extract a step-wise explanation from a proof, with extra attention to minimization techniques that keep the step-wise explanations small (and comprehensible) and;
    \item we evaluate our approach on several application domains, 
    showing our approach computes step-wise explanations up to 100 times faster, with a limited effect on the length of the explanation and its individual steps.
\end{itemize}

\newcommand\CSP{(\variables, \domains, \constraints)\xspace}

\section{Background}

A Constraint Satisfaction Problem (CSP) is a triple $\CSP$ with $\variables$ a set of integer variables $\domains$ a set of domains (one for each variable), and $\constraints$ a set of constraints \cite{RvW06HandbookConstraintProgramming}.
When the set of variables and domains is clear from the context, we write $\constraints$ instead of $\CSP$.
Boolean variables are represented as 0-1 integer variables.
A constraint in this paper is \textit{any} Boolean expression over $\variables$, which maps variable assignments to true or false, and is written as a mathematical expression.
E.g., $x + y \leq 3$ and $(x = 3) \vee (y \geq 2)$ both represent a constraint.
Variables that occur in the arguments of a constraint are said to be in the constraint's \emph{scope}.
We write $\scope{c}$ to denote the set of variables occurring in $c$.
The scope trivially generalizes to sets of constraints.
An assignment $\alpha$ \emph{satisfies} a constraint when the constraint maps it to true.
An assignment that satisfies every constraint in $\constraints$ is a \emph{solution} of the CSP.
If no solution to the CSP exists, it is said to be \emph{unsatisfiable}.
The set of solutions to a CSP is written as $\sols{\constraints}$.
The set of solutions to a CSP, projected to a (sub)-set of variables $\variablesV$ is written as $\solsto{\constraints}{\variablesV}$.
A set of constraints $\constraints'$ is \emph{implied by} 
$\constraints$ if every solution of $\constraints$ is also a solution of $\constraints'$. I.e., $\sols{\constraints} \subseteq \sols{\constraints'}$.
We write this as $\constraints \models \constraints'$, and call $\constraints'$ a \emph{logical consequence} of $\constraints$.
Now, $\constraints'$ is a logical consequence of $\constraints$ if and only if $\constraints \cup \set{\bigvee_{c' \in \constraints'} \neg c'}$ is unsatisfiable. 

Typically, a user \emph{models} a \emph{user-level} CSP in a constraint modeling system, such as \minizinc~\mycite{minizinc} or \cpmpy~\cite{cpmpy}, where they can use any Boolean expression as a constraint.
This allows the user to focus on the modeling task at hand, without having to take the specifics of the solver into consideration.
Next, they choose a general-purpose solver to actually solve the problem. 
To enable this, the modeling system \emph{transforms} the \textit{user-level} CSP to an equivalent, \emph{solver-level} CSP, using only constraints that are directly supported by the solver (e.g., using flat (global) constraints for CP solvers, or using linear inequalities for MIP solvers~\cite{BSTW16ImprovedLinearizationConstraintProgrammingModels}).
During this transformation, the modeling system may introduce \emph{auxiliary variables}.
Such variables do not represent an entity in the \emph{user-level} CSP, and a user typically does not care about their value.
We assume the transformations implemented in the constraint modeling system do not remove nor introduce new solutions when projected to the original decision variables, 
i.e., $\solsto{\constraints}{\variables} = \solsto{\mathit{transform}(\constraints)}{\variables}$.
\ifarxiv
\Cref{example:transform} in the appendix shows an example of such a transformation.
\fi
When providing the user with an explanation, this should be in terms of \emph{user-level} variables and constraints.
E.g., the explanation should not contain any statement regarding auxiliary variables, as the user does not know their meaning.

CP solvers solve a CSP using \emph{propagation-based solving}.
That is, the solver \emph{filters} or \emph{propagates} the domains of the decision variables based on the semantics of the constraint.
Propagators can have different levels of consistency, with \emph{domain consistent} propagators being the most powerful in terms of propagation strength, often at the cost of a higher runtime of the propagator~\cite{B06ConstraintPropagation}.
We write the propagation function of a solver-level constraint $R$ as $f_R$.

Throughout this paper, we will often be interested in subsets of constraints that are unsatisfiable.
We formalize such subsets using the notion of a Minimal Unsatisfiable Subset (MUS)~\cite{lynce2004computing}.
\begin{definition}[Minimal Unsatisfiable Subset]
Given a partitioning of a set of constraints $\constraints$ into soft constraints $\softcons$ and hard constraints $\hardcons$, a Minimal Unsatisfiable Subset (MUS) is a subset $M \subseteq \softcons$ such that $M \cup \hardcons$ is unsatisfiable and for any strict subset $M' \subsetneq M$, $M' \cup \hardcons$ is satisfiable.    
\end{definition}

We write $\muscall{\softcons}{\hardcons}$ to indicate we compute an MUS for soft constraints $\softcons$ and hard constraints $\hardcons$.

\section{A Framework for Proofs of Unsatisfiability}

As a first contribution, we introduce a unifying framework to describe a proof of why a CSP is unsatisfiable.
This framework will allow us to specify differences and similarities between solver-generated proof logs and user-oriented step-wise explanations.
We re-interpret and formalize both using the uniform concept of \emph{abstract proofs}.
An abstract proof is a sequence of proof steps where each proof step derives new constraints $\cons$, using a set of reasons $\reason$ which are either part of the input CSP or are previously derived in the proof.

\begin{definition}[Abstract proof]
Given a CSP $\CSP$, an \emph{abstract proof} is a sequence of pairs $\stepp{\cons_i}{\reason_i}$ with $\cons_i$ any set of constraints derived by the proof step, and $\reason_i$ their reasons, another set of constraints.
An abstract proof is valid if for each step, the derived constraint is implied by its reasons ($\reason_i \models \cons_i$) and for each step $i$, $\reason_i \subseteq \constraints \cup \bigcup_{1 \leq j < i} \cons_j$.
\end{definition}

We will omit curly brackets when the derived set of constraints is a singleton set.
An abstract proof of length $n$ proves the unsatisfiability of the CSP if $\bot \in C_n$, i.e., if the last step in the proof derives the ``false'' literal.

The concept of an abstract proof is, of course, very general, and any proof of unsatisfiability can be represented in it.
However, by imposing different restrictions on the content of the proof-steps, an abstract proof instantiates a Deletion Reverse Constraint Propagation (\pformat) prooflog or a step-wise explanation sequence, as we will show next.

\begin{subsection}{DRCP as abstract proof}
If a constraint solver concludes that a solver-level CSP is unsatisfiable, it may produce a proof supporting that conclusion. 
Here we describe the \pformat format~\cite{FSIJE24MultiStageProofLoggingFramework}, which was designed to represent the behavior of Lazy Clause Generation Constraint Programming solvers.

The smallest unit used in a \pformat proof is the atomic constraint: a statement about the domain of a single variable.
\begin{definition}[Atomic Constraint]
    An atomic constraint ${\atomic{x \diamond v}}$ is a comparison between a variable and a constant.
    I.e., $x \in \variables$, $\diamond\in\set{\leq, \geq, =, \neq}$ and $v \in \integers$.
\end{definition}

Each step in a \pformat proof may be interpreted as an abstract proof-step $\stepp{\clause}{\reason}$ with \clause a \emph{disjunction} of  \emph{atomic constraints}, which we will call a \emph{clause}.
Based on the content of $\reason$, each step in the proof can be categorized as an \emph{inference} or a \emph{nogood-deriving step}.

\begin{definition}[Inference proof step]
A proof step $\stepp{\clause}{\reason}$ is an inference step if $\clause$ is a clause of atomic constraints and $\reason$ is a singleton set with a solver-level constraint.
\end{definition}
Inference proof-steps represent the result of a \emph{constraint propagation}.
To allow the proof to be efficiently checked by an external verifier, an inference step also includes which filtering algorithm was used to derive \clause from $\reason$. 
An inference is valid with respect to propagator $f_\reason$ if $f_\reason(\neg \clause) = \bot$.

\begin{definition}[Nogood deriving step]
A proof step $\stepp{\clause}{\reason}$ is a nogood learning step if $\clause$ is a clause of atomic constraints, and $\reason$ is a set of previously derived clauses.
\end{definition}
A nogood-deriving step represents a clause (sometimes also called nogood) that is \emph{learned} by the solver, after it found a conflict in a branch of the search-tree~\cite{MLM21Conflict-DrivenClauseLearningSATSolvers,LSTReasoningFromLastConflictConstraintProgramming}.

Sometimes, not all steps stored in a proof are strictly necessary.
This is, for example, the case when a particular step is never used as a \emph{reason} to derive a later step.
In a trimmed proof, as defined next, such steps do not occur.

\begin{definition}[Trimmed abstract proof]
    An abstract proof $[\stepp{\cons_1}{\reason_1}, \stepp{\cons_2}{\reason_2}, \dots, \stepp{\bot}{\reason_n}]$ is trimmed, if each $c \in \cons_i$ appears in at least one reason $R_j$ later in the proof.
\end{definition}

Trimming can be implemented efficiently when the proof format stores the set of reasons used to derive a constraint explicitly~\cite{CHHKS17EfficientCertifiedRATVerification}, such as in DRCP.

\end{subsection}

\subsection{Step-wise explanations as abstract proofs}

The \emph{step-wise explanation framework} was introduced by \citet{BGG21frameworkstep-wiseexplaininghowsolveconstraint} for explaining how to find the unique solution of a logic puzzle using simple derivations and was recently extended to be used in the context of explaining unsatisfiable CSPs~\cite{BDGBG23SimplifyingStep-WiseExplanationSequences}.
A step-wise explanation of a user-level CSP is an abstract proof where in each proof step $\stepp{\cons}{\reason}$, the derived constraint is a set of atomic constraints (interpreted as a conjunction); and the set of reasons consists of \emph{user-level} constraints, and constraints derived by earlier steps.

\begin{definition}[Explanation step] 
Given a user-level CSP $(\variables,\domains,\constraints)$, an \emph{explanation step} is an abstract proof step $\stepp{\cons_i}{\reason_i}$ where $\cons_i$ is a set of atomic constraints.
\end{definition}

Typically, we are not interested in arbitrary explanation steps, but have a preference for \emph{simple} steps.
That is, steps in which only a few user-level constraints from the CSP are required to derive the set of new atomic constraints.
Most algorithms for computing explanation steps find a minimum set of user constraints to use in a single step.

\begin{example}
\Cref{fig:minimize} shows three alternative explanation steps for a Sudoku, explaining why the green cell must be a 4 (i.e., $\cons = \cell_{3,1} = 4$.).
In the leftmost step, its reasons are $\reason = \set{\alldiff{\mathit{row}_3}, \alldiff{\mathit{col}_2}, \cell_{2,2} = 4,\dots}$
\end{example}

We can build a step-wise explanation sequence by combining such explanation steps.
In the literature, several methods for generating such explanation sequences exist~\cite{BGG21frameworkstep-wiseexplaininghowsolveconstraint,BDGBG23SimplifyingStep-WiseExplanationSequences,GBG23EfficientlyExplainingCSPsUnsatisfiableSubsetOptimization}.
However, they suffer from scalability issues when many \emph{facts} can or need to be explained, or when many constraints need to be used in a single step.

To summarize, the main difference between DRCP-proof steps and explanation steps are 1) the type of constraints used as reasons (solver-level vs user-level) and 2) the type of constraints derived in a proof step (disjunction vs conjunction).
In the next section, we will explore methods that allow us to transform a DRCP-proof into a step-wise explanation.

\begin{section}{DRCP-proof to Step-wise explanation} \label{sec:convert}
We now present our second contribution and propose a set of techniques for processing abstract proofs.
By applying these techniques, an abstract proof corresponding to a \pformat proof can be transformed into a step-wise explanation.

\subsection{Proof simplification} \label{sec:squashing}
To transform a DRCP proof into a step-wise explanation, we will need to gradually restrict what is allowed in the proof, and rewrite an abstract proof into an equivalent restricted proof without altering the validity of the proof. 
For this, we introduce the concept of an abstract P-proof:
\begin{definition}[Abstract P-proof]
An abstract P-proof is an abstract proof \mathabstractproof{n} where property P holds for every step $\stepp{\cons_i}{ \reason_i}$ in the proof.
\end{definition}

If a proof step $\stepp{\cons}{\reason}$ does not satisfy property $P$, we remove it from the proof, and update all steps involving $\cons$ in its reason.
That is, any other step $\stepp{C'}{R'}$ which uses $\cons$ in its reason, (i.e., $C \cap R' \neq \emptyset$), can be replaced by the step $\stepp{C'}{(R' \setminus C) \cup {R}}$.
In what follows, we illustrate two cases where we can apply this simplification to transform a proof log into a step-wise explanation.

\paragraph{Simplification over auxiliary variables}

In a step-wise explanation, each step derives atomic constraints about the variables in the user-level CSP.
However, as a DRCP proof is a proof of the \emph{solver-level} CSP, its derived constraints may contain auxiliary variables introduced by the modeling system.
Using proof simplification, we remove any proof step that derives a constraint over such auxiliary variables.
That is, we convert the proof to an abstract P-proof with $P(\stepp{\cons}{\reason}) \coloneq \scope{C} \subseteq \variables$. 

\begin{example}[Removing a step with an auxiliary variable]
Consider the following constraint $a_1 \implies x+4 \leq y$ as a result of transforming the user-level constraint $(x + 4 \leq y) \vee (y + 6 \leq x)$.
A proof-step involving this constraint is $\stepp{a_1 = 0 \vee x \neq 3 \vee y \neq 2}{\set{a_1 \implies x + 4 \geq y}}$.
Here, $a_1$ was not part of the input CSP and was introduced by the modeling system as part of a transformation.
Hence, we remove the step, and in any later proof step, we replace the derived clause with the constraint $a_1 \implies x+4 \leq y$.
\end{example}

For abstract proofs certifying unsatisfiability, this is always possible as $\scope{\bot} = \emptyset \subseteq \variables$.
Hence, in the unlikely case where each proof step derives a constraint involving an auxiliary variable, the proof will be simplified to a trivial proof with one step, containing a set of solver-level constraints that derive $\bot$.

\paragraph{Simplification to domain reductions}
In the step-wise explanation framework, each intermediate result is a set (i.e., a conjunction) of atomic constraints.
That is, the derived constraints in each explanation step describe a set of allowed domain values for the variables in the CSP.
An abstract proof originating from a \pformat proof may also contain such intermediate results describing a domain.
In particular, any clause of atomic constraints over the same variable in the CSP describes a domain.
Hence, any step in a proof that derives such a clause can be part of a step-wise explanation.
We convert the abstract proof to an abstract P-proof, with $P(\stepp{\cons}{\reason}) \coloneq |\scope{\cons}| \leq 1$. 
For proofs of unsatisfiability, this is possible as $\scope{\bot} = \emptyset$.
Thus, in the unlikely case no proof step in the given proof contains a single domain reduction as a derived constraint, the proof can again be simplified to a trivial proof with 1 proof step.

\subsection{Solver-level inferences to user-level inferences} \label{sec:inferences}
Each \emph{inference} step in a solver-generated proof is a \emph{solver-level constraint}, e.g., a linear inequality or a clause, depending on the proof format~\cite{WHH14DRAT-trimEfficientCheckingTrimmingUsingExpressive}.
In a step-wise explanation sequence, each step consists of a set of \emph{user-level} constraints.
However, multiple solver-level constraints may correspond to one user-level constraint. 

By keeping track of which user-level constraint \ignore{constraints} generate a certain solver-level constraint, we can replace any solver-level constraints in the reasons of a proof, if the derived constraint $\cons$ does not involve an auxiliary variable.
Indeed, if the user-level constraint is \emph{stronger} than or equally strong as the solver-level constraint given decision variables $\variables$, we can safely replace the solver-level constraint with the user-level constraint, as formalized by the generalized Lemma \ref{lemma:userlevel}.

\begin{lemma}[Validity of user-level constraints in proof steps]\label{lemma:userlevel}
    Given a valid proof step $\stepp{\cons}{\reason}$ with $s \in \reason$, and $\scope{\cons} \subseteq \variables$.
    If $\solsto{c}{\variables} \subseteq \solsto{s}{\variables}$, then ${\stepp{\cons}{\reason \setminus \set{s} \cup \set{c}}}$ is a valid proof step.
\end{lemma}

\begin{example}[Replacing a solver-level constraint]
Consider the constraint $x \neq y$ as part of the decomposition of the user-level global constraint $\alldiff{x,y,z}$.
The step $\stepp{(x \neq 1) \vee (y \neq 1)}{\set{x \neq y}}$ can then safely be replaced with the step $\stepp{(x \neq 1) \vee (y \neq 1)}{\set{\alldiff{x,y,z}}}$.
\end{example}

Note, however, that this simplification is only valid if no auxiliary variables occur in the derived constraint $\cons$.
Indeed, a user-level constraint cannot logically imply a statement about an auxiliary variable introduced by the modeling system.
This means that each proof step $\stepp{\cons}{\reason}$ where an auxiliary variable is part of $\cons$ should be simplified away \textit{before} replacing solver-level constraints with a user-level ones in the proof.
This can be done using the proof simplification technique described in \Cref{sec:squashing}.

\subsection{Reason minimization} \label{sec:implicationgraph}
In any abstract proof, including a step-wise explanation, some proof steps may contain more information than strictly necessary.
More precisely, some \emph{reasons} may be redundant for deriving a particular derived constraint.
In a user-oriented step-wise explanation, such redundant information is unwanted, as it requires more cognitive effort to comprehend.
In \Cref{algo:minimize}, we present a method that minimizes the set of reasons for each step in an abstract proof.

\newcommand{\mode}{\mathit{mode}}
\newcommand{\cand}{\mathit{cand}}

\begin{algorithm}
\caption{\call{MinimizeReasons}}
\label{algo:minimize}
\begin{algorithmic}[1]
    \State \textbf{Input:} proof \mathabstractproof{n}, $\mode$
    \State \textbf{Output:} trimmed abstract proof with subset-minimal reasons
    \State $\mathcal{P} \gets []; \req \gets \cons_n;$
    \For {$i = n..1$} \Comment{from back to front}
        \If{$\cons_i \cap \req = \emptyset$} \textbf{continue} \EndIf
        \If{$\mode =$ local}
        $\cand \gets \reason_i$
        \EndIf
        \If{$\mode =$ global}
        $\cand \gets \constraints \cup \bigcup_{j < i} \cons_j $
        \EndIf
        \State $\reason_i' \gets \muscall{\cand}{\neg\cons_i}$ \label{line:muscall}
        \State $\req \gets \req \cup \reason_i'$
        \State Add step $\stepp{\cons_i}{\reason_i'}$ at the start of $\mathcal{P}$
    \EndFor
\State \Return $\mathcal{P}$
\end{algorithmic}
\end{algorithm}

The main idea of the algorithm is to traverse the proof from back to front, and keep a set of \emph{required} ($\req$) reasons.
That is, $\req$ contains at any point the set of reasons that have to be derived by previous steps.
If we encounter a step that is not required, it can be discarded from the proof.
Once the set of required reasons is empty, we are certain that all steps in the proof have a minimal set of reasons.

\begin{figure*}[ht]
\input{pipeline_aaai}
\caption{Overview of the pipeline for generating a step-wise explanation sequence from a proof. }
\label{fig:pipeline}
\end{figure*}
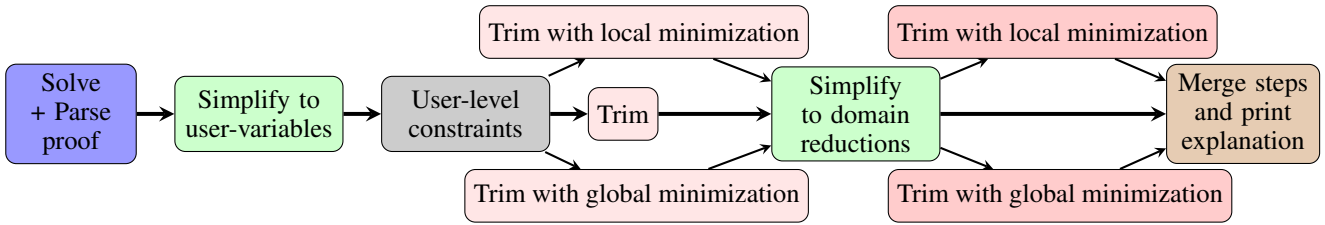

For each step, the set of minimal reasons is computed using an MUS algorithm (line \ref{line:muscall}).
Our algorithm is agnostic to the exact method that is used for computing MUSes, and several algorithms for computing such MUSes are available in the literature~\cite{M10MinimalUnsatisfiabilityModelsAlgorithmsApplicationsInvited}.
We use an algorithm that finds an MUS which minimizes the set of user-level constraints~\cite{IPLM15SmallestMUSExtractionMinimalHittingSet}.
We consider two versions of minimization \cref{algo:minimize}: local or global.
 
With \emph{local minimization}, we only consider the set of given reasons of a proof step as candidate reasons for the minimized step.
This mode of the algorithm can be interpreted as a strict \emph{trimming} algorithm as it is agnostic to the inference algorithms used to construct the proof.
Indeed, while the solver may require several input constraints or intermediate steps to derive a new constraint, not all of these steps or constraints may be strictly required.
This form of trimming can considerably reduce proof size, especially when the propagation algorithms for some constraints are not domain consistent, or when a user-level constraint is \textit{decomposed} into several smaller constraints before solving.

With \emph{global} minimization, we consider \emph{all} user-level constraints in the CSP, and previously derived constraints as candidate reasons.
This setting has the advantage that the algorithm is less restricted in its choice of reasons, potentially leading to shorter or simpler proofs.
However, this comes at the potential downside of higher computation times of the MUS algorithm.
Indeed, because the set of candidates is larger, computing an MUS is conceptually more expensive compared to the local version of the algorithm.

\begin{example}[Local and Global Minimization]
Consider the middle proof step in \Cref{fig:minimize}.
This step needlessly highlights the top-right 3. 
Local minimization will always exclude the constraint $\cell_{1,4} = 3$ from its reasons.

The rightmost proof step cannot be minimized locally as it already uses a subset-minimal set of reasons.
However, by considering all possible reasons, we can minimize it to the leftmost proof-step, where we use \underline{different} Sudoku rules, while deriving the same constraint $\cell_{3,1} = 4$.
\end{example}

\newcommand{\figwidth}{0.26\linewidth}
\begin{figure}[ht]
    \centering
    \begin{minipage}[b]{\figwidth}
         \includegraphics[width=\textwidth]{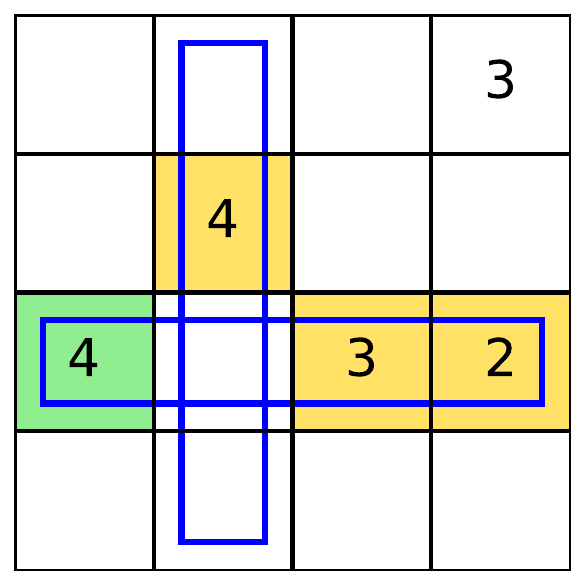}
    \end{minipage}
    \begin{minipage}[b]{\figwidth}
    \includegraphics[width=\textwidth]{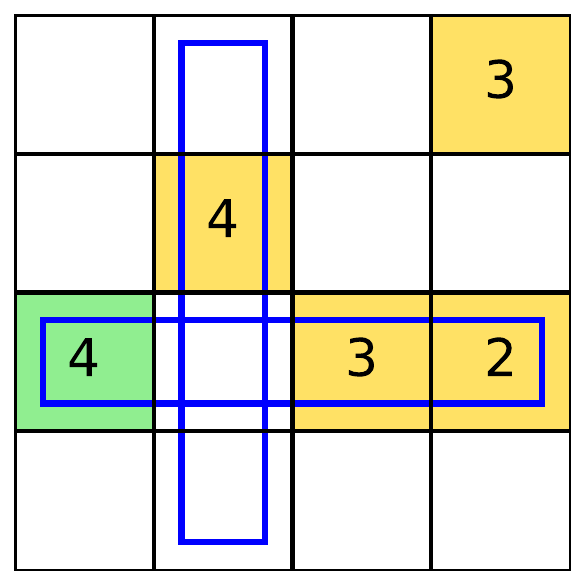}
    \end{minipage}
    \begin{minipage}[b]{\figwidth}
        \centering
        \includegraphics[width=\textwidth]{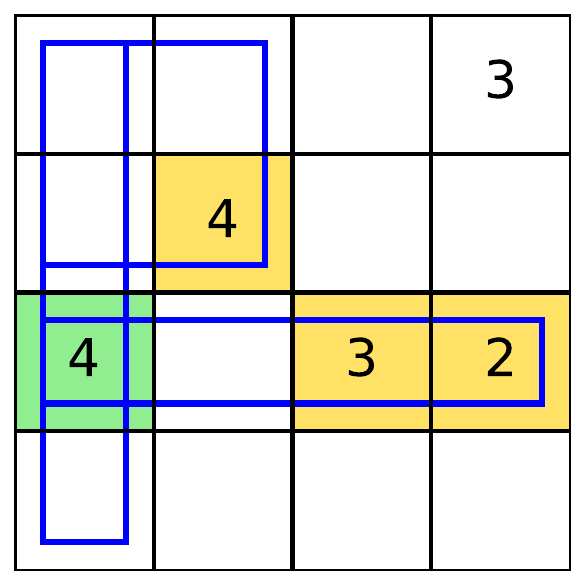}
    \end{minipage}
    \caption{Alternative explanation steps for a 4x4 Sudoku}
    \label{fig:minimize}
\end{figure}

\subsection{Overview of the proposed methods}
\Cref{fig:pipeline} shows the overview of our proposed pipeline, which allows to transform a DRCP-formatted proof log into a step-wise explanation.
The minimal set of processing techniques required is represented as the center line in the diagram.

We first transform the proof log into a user-level proof by removing all proof steps that contain auxiliary variables in the derived constraint using \textbf{proof simplification}. 
Then, we replace all solver-level constraints in the reasons of the proof steps with their user-level constraint it was derived from.
We then \textbf{trim} the proof to remove any redundant steps.
From this smaller proof, we remove all proof steps that derive an intermediate constraint over multiple variables.
Finally, we merge proof steps that have the same set of reasons to a single step, and return this step-wise explanation sequence to the user.
Optionally, we can apply more elaborate \textbf{reason minimization}.
Local or global minimization can be applied either directly after the proof is transformed to the user level, or after the proof is simplified to only contain domain reductions.
In general, we expect better step-wise explanations when minimizing at the end of the pipeline, as after domain reductions are removed, the proof is already a step-wise explanation.
Hence, by applying further minimization, we directly impact the final quality of the step-wise explanation.
In contrast, minimizing before the final simplification stage in the pipeline can be faster as each proof step has fewer reasons to consider during the optional, final minimization.

Interestingly, our approach using \emph{global minimization} at the end of the pipeline shares some similarity with the algorithm of \citet{GBG23EfficientlyExplainingCSPsUnsatisfiableSubsetOptimization}.
To compute the next explanation step in the sequence, they iterate over \emph{each} fact to explain and find an \emph{optimal} explanation using an MUS algorithm.
Then, across all those facts, the optimal explanation is chosen and added to the sequence as a new step.
Crucially, the algorithm of \citet{GBG23EfficientlyExplainingCSPsUnsatisfiableSubsetOptimization} does not scale well when many facts can be explained, as it needs to optimally compute explanations for each remaining fact at every step.
Hence, their method is more suited to a setting where only few possible facts need explaining, such as the unique solution of a logic puzzle.
In contrast, when explaining unsatisfiable models, \textit{any} atomic constraint over any variable is a candidate fact.
Our approach using global minimization at the end of the pipeline essentially determines in one go which \emph{fact} should be explained in which part of the sequence, using the order of derived constraints in the proof, and uses an MUS algorithm only to find an optimal set of reasons to explain those facts in turn.

\newcommand{\figheight}{125pt}
\newcommand{\legendwidth}{91pt}
\begin{figure*}[ht]
    \begin{minipage}[c]{\linewidth-\legendwidth-2pt}
    \includegraphics[height=\figheight]{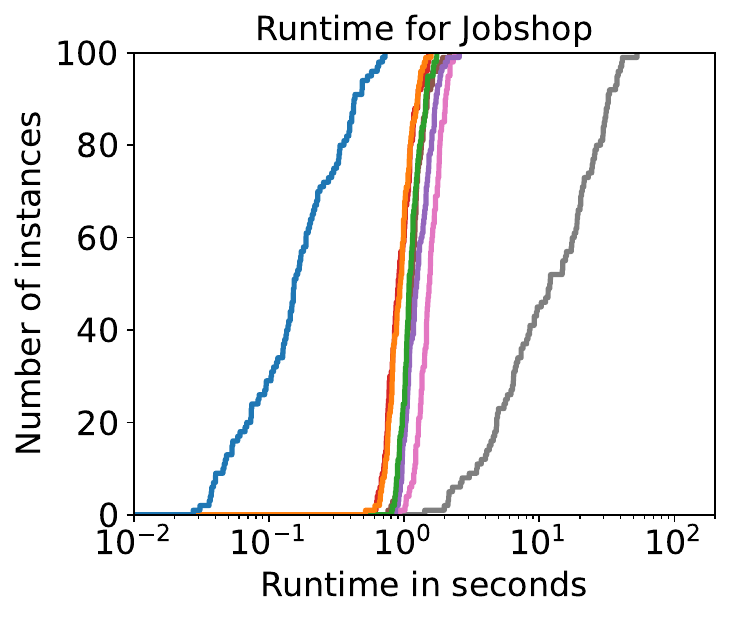}
    \includegraphics[height=\figheight]{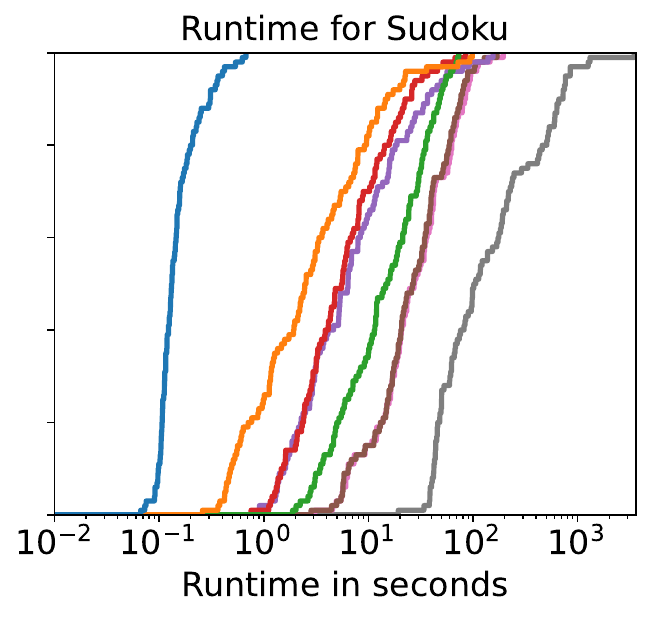}
    \includegraphics[height=\figheight]{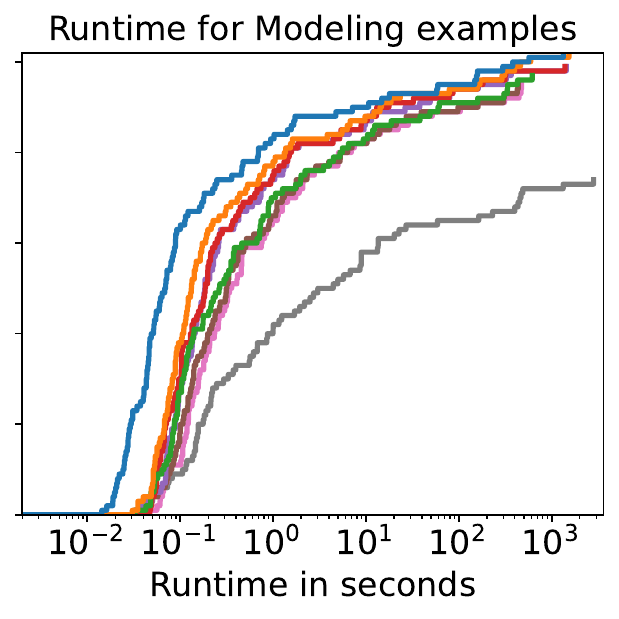}
    \end{minipage}
    \begin{minipage}[c]{\legendwidth}
        \includegraphics[width=\textwidth]{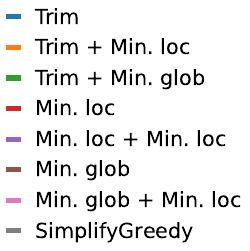}    
    \end{minipage}
    \caption{Runtime of all methods for different benchmarks}
    \label{fig:ecdf}
\end{figure*}

\section{Experimental evaluation \pageplan{Top of page 6}} \label{sec:experiments}
We answer the following experimental questions:
\begin{EQ}
    \item How does the runtime of existing approaches compare to our methods?
    \item How is the runtime of our approach affected by each combination of proof-minimization?
    \item How is the quality of step-wise explanation impacted by each combination of proof-minimization?    
\end{EQ}

\paragraph{Experimental setup}
We implement our approach in Python 3.11 using the \cpmpy~\cite{cpmpy} library v0.9.24.
To generate proofs, we use the Pumpkin LCG-solver on commit \texttt{d730b05}.
For computing MUSes, we employ the SMUS algorithm~\cite{IPLM15SmallestMUSExtractionMinimalHittingSet}, with Exact v2.1.0~\cite{Exact} as a SAT-oracle and Gurobi v.12.0.1 as a hitting set solver.
As hardware, we used a single core of an Intel(R) Xeon(R) Silver 4514Y and 8GB RAM.

\paragraph{Benchmarks}
We use three benchmark families, similar to~\citet{BDGBG23SimplifyingStep-WiseExplanationSequences}.
We use 100 unsatisfiable Sudoku instances (\benchsudoku), 100 job-shop instances with a limited makespan in order to make the model unsatisfiable (\benchjobshop), and a diverse set of 102 example models from the CPMpy and CSP-lib repositories (\benchdebug), which are altered to contain a modeling error, such as an off-by-one error or swapping a comparison.

\paragraph{Methods under investigation}
We compare each variant of our approach corresponding to each of the paths in \Cref{fig:pipeline}: without any optional steps (\approachtrim), with local or global minimization at the end of the pipeline (\approachtrimloc or \approachtrimglob), with local or global minimization instead of simple trimming (\approachloc or \approachglob), with and without local minimization at the end of the pipeline (\approachlocloc and \approachglobloc).
We do not use any other minimization method in combination with global minimization at the end of the pipeline, as it would ignore the dependencies uncovered by the first minimization anyway.
As the current state-of-the-art for computing step-wise explanations of unsatisfiable CSPs, we consider the approach by \citet{BDGBG23SimplifyingStep-WiseExplanationSequences} (\sota).
We measure the wall-clock time taken by the entire process when using each of these methods and answer the experiment questions.

\subsection{EQ1: Comparison of runtime to literature}
\Cref{fig:ecdf} shows a cumulative distribution plot for all of the methods, for each benchmark.
For both the \benchsudoku and \benchjobshop benchmarks, all methods are able to find a step-wise explanation for all instances.
However, compared to \sota, our approach with minimal post-processing (\approachtrim) finds a step-wise explanation ±100x faster.
This clearly shows that the overhead of generating the proof during solving once, is negligible compared to the repeated calls to an NP-oracle as done by \sota.
Indeed, our minimal \approachtrim approach does not require \emph{any} NP-reasoning after the solver has produced the proof, as the DRCP format lists the reasons each step explicitly in the proof file.
In the \benchdebug benchmark, \sota is unable to find a step-wise explanation for all instances, and runs out of time (1h) or memory for 29 of the 102 instances.
For all benchmarks, even our slowest method (\approachglobloc) finds an explanation sequence at least 10 times faster compared to the state of the art.

\begin{table*}[ht]
    \centering
\input{results_table}
    \caption{Quality of final explanation sequence, computed on instances finished by all, shown for non-dominated methods.}
    \label{table:results}
\end{table*}
\subsection{EQ2: Runtime of minimization techniques}
As expected, our approach with minimal processing, \approachtrim, is overall faster compared to any other method we tested.
In general, \textit{local minimization} is faster compared to \textit{global minimization}.
This is not surprising as in global minimization, many more candidate reasons may have to be considered to compute the proof step.
The runtime difference is especially noticeable in the \benchsudoku benchmark, where, for example, \approachtrimloc finds an explanation sequence around 5-10 times faster compared to \approachtrimglob.
Our approach allows us to minimize each proof step either before and/or after simplifying the proof to only contain domain reductions.
In general, it is faster to minimize after this simplification.
This means the simplification of the proof to only contain domain reductions significantly shortens the proof, compared to before.
Indeed, starting with a shorter proof before minimization is faster to process, as for each step, an MUS algorithm has to be called to find a minimal set of reasons.
The difference is clear in the \benchjobshop benchmark 
\ifarxiv
(see \Cref{tab:proofsize:jobshop,tab:proofsize:sudoku,tab:proofsize:debug}),
\fi
as scheduling problems are by nature very disjunctive, and hence, intermediate constraints derived in the proof typically do not describe a domain.

\begin{figure}[t]
    \centering
    \includegraphics[width=0.9\linewidth]{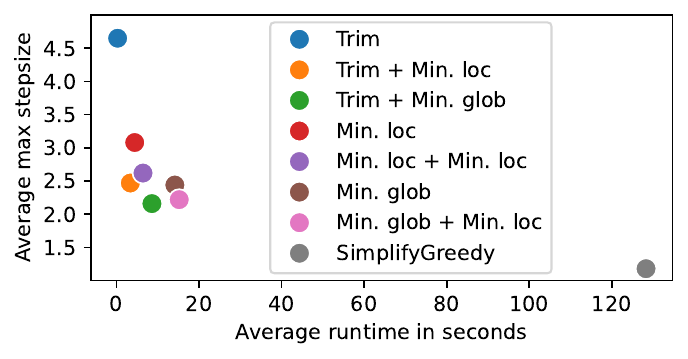}
    \caption{Tradeoff of quality vs runtime on all benchmarks.}
    \label{fig:tradeoff}
\end{figure}

\subsection{EQ3: explanation quality}
\Cref{table:results} lists the quality of a step-wise explanation based on its most difficult step in terms of the number of user-level constraints used, and the overall length of the sequence.
Note these metrics are correlated, as a sequence using larger steps often needs fewer steps overall to prove the conclusion.
\Cref{fig:tradeoff} plots the runtime and explanation quality of each of the proposed methods.
We are interested in methods on the lower left of the plot, i.e., one that is \emph{fast} and computes a \emph{simple} step-wise explanation.
Our methods with reason minimization determine a trade-off of explanation quality and runtime.
We now compare each of these methods in depth.

Overall, our approach without any minimization (\approachtrim) produces the most complex explanation sequences.
For example, in the \benchjobshop benchmark, all other methods are able to find explanation sequences with only a single constraint in each step, whereas \approachtrim needs on average 2.7.
This is also the case for \benchsudoku and \benchdebug where \approachtrim needs on average 7.5 and 3.5 constraints, respectively.

Sequences produced by \approachtrimloc with local minimization at the end of the pipeline use significantly fewer constraints in their most difficult step.
Interestingly, using local minimization instead of simple trimming (\approachloc) also improves the explanation quality for all benchmarks.
This means that in most proofs produced by the solver, the set of reasons provided in the proof contains many redundant steps.
Indeed, while the solver may require many inference steps to derive a particular proof step, this is not the case after minimizing, as we only require the derived constraint to be logically implied.

Our best overall method is the \approachtrimglob method, which produces explanation sequences of similar quality to those by \sota, with a significantly lower runtime.
Note that for \benchsudoku, the higher average complexity of the explanation sequence is mainly due to outliers in the data.
Indeed, for more than half of the explanation sequences, it produces sequences with only a single constraint in each proof step.
Still, on average, \sota finds explanation sequences with smaller steps for \benchsudoku and \benchdebug.
This means there are \textit{facts} that are \textit{easier} to explain compared to those listed in the DRCP proof.

To summarize, while \approachtrim finds an explanation very fast, the final quality is worse than with other methods.
\sota produces the simplest sequences, but at the cost of a significantly higher runtime.
Our method \approachglob, computes step-wise explanations that are often of similar quality as \sota, but does so in a fraction of the time.

\section{Conclusion and Future Work}

We proposed a method for computing step-wise explanations from proofs generated by a certifying constraint solver.
By representing \pformat formatted proofs and step-wise explanation sequences in a common framework of \emph{abstract proofs}, we have shown how to convert a proof designed to be \emph{machine verifiable} into a step-wise explanation for a \emph{human user}.
We have shown several techniques that can remove solver-specific information, such as auxiliary variables and solver-level constraints from the proof, and proposed a set of minimization techniques to further optimize the explanation quality.
Different combinations of these minimization techniques result in different tradeoffs between runtime and quality of the final explanation sequence.
Our best approach produces explanation sequences of similar quality to the state-of-the-art, with 10 times lower runtime.
Our methods can now be applied to larger problems, for which it is unthinkable to use the existing algorithms.

We see two main directions for future work.
Firstly, we would like to further reduce the size of the step-wise explanations computed by our approach.
For example, by finding shorter proofs~\cite{SvdLCGWDShirtShoterShortestProofsUnsatisfiability} or by adapting the solver behavior to find proofs better suited to transform into a step-wise explanation.
Secondly, while the goal of this paper is to find explanations in the already existing format of step-wise explanation sequences, we believe more inspiration from proof logging and proof-theory can be found to improve the actual \emph{explanation}.
E.g., while we remove auxiliary variables as they do not fit the step-wise explanation framework, extra variables may be useful to compress the overall proof.
However, auxiliary variables can be introduced by the modeling system in many different ways.
Therefore, the question of how to integrate auxiliary variables in a user-level explanation is still an open and challenging problem.

\end{section}

\section*{Acknowledgments}

This work is partially funded by the European Union (ERC, CertiFOX, 101122653; ERC, CHAT-Opt, 01002802 and Europe Research and Innovation program TUPLES, 101070149). Views and opinions expressed are however those of the author(s) only and do not necessarily reflect those of the European Union or the European Research Council. Neither the European Union nor the granting authority can be held responsible for them.

This work is also partially funded by the Fonds Wetenschappelijk Onderzoek -- Vlaanderen (projects G064925N and G070521N).

Maarten Flippo is supported by the project ”Towards a Unification of AI-Based Solving Paradigms for Combinatorial Optimisation” (OCENW.M.21.078) of the research programme ”Open Competition Domain Science - M” which is financed by the Dutch Research Council (NWO).

\bibliography{aaai26_used_refs.bib}

\ifarxiv
\setcounter{secnumdepth}{1}

\appendix

\section{Additional examples and proofs}
Below are several examples referenced throughout the main text of the paper, showing some concepts or algorithms in more detail.

\begin{example}[Transforming constraints]\label{example:transform}
Consider the compound constraint $(x + 4 \leq y) \vee (y + 6 \leq x)$.
This constraint is typically not supported by a solver, and can be transformed to simpler constraints $\set{a_1 \implies y - x \geq 4, a_2 \implies x - y \geq 6, a_1 \vee a_2}$ with $a_1$ and $a_2$ newly introduced auxiliary variables.
\end{example}

\newcommand{\derivedone}{x > 0 \vee y < 2 \vee z \geq 1}
\newcommand{\derivedtwo}{x > 0 \vee y < 2 \vee z \leq 0}
\begin{example}[Inference step]\label{example:inference}
Consider the constraint $2x - y + 2z \geq 0$.
A valid inference step that uses this constraint is
$\stepp{\derivedone}{2x - y + 2z \geq 0}$.
This step could be derived after the solver branched on $x \leq 0$ and $y \geq 2$, or when those two atomic constraints became true through other propagations.
\end{example}

\begin{example}[Nogood deriving step]\label{example:nogood}
Consider the derived constraint from \Cref{example:inference} and suppose another inference step derived the clause $\derivedtwo$.
Then, $\stepp{x > 0 \vee y < 2}{\set{\derivedone, \derivedtwo}}$ is a valid nogood-learning step.
\end{example}

\setcounter{lemma}{0}
\begin{lemma}[Validity of user-level proof steps]
    Given a valid proof step $\stepp{\cons}{\reason}$ with $s \in \reason$, and $\scope{\cons} \subseteq \variables$.
    If $\solsto{c}{\variables} \subseteq \solsto{s}{\variables}$, then $\stepp{\cons}{\reason \setminus \set{s} \cup \set{c}}$ is a valid proof step.
\end{lemma}
\begin{proof}
Given that $\stepp{\cons}{\reason}$ is a valid proof step, we know that $\reason \models \cons$.
Because $\cons$ does not contain any auxiliary variables, this is equivalent to $\solsto{\reason}{\variables} \subseteq \solsto{\cons}{\variables}$.
Additionally, as $\solsto{c}{\variables} \subseteq \solsto{s}{\variables}$, it holds that $\solsto{\reason \setminus \set{s} \cup \set{c}}{\variables} \subseteq \solsto{\reason}{\variables}$.
Combining this, we get that $\solsto{\reason \setminus \set{s} \cup \set{c}}{\variables} \subseteq \solsto{\cons}{\variables}$ and thus, the rewritten proof step is valid.
\end{proof}

\section{Benchmarks}\label{appendix:benchmarks}
We describe the exact details of the benchmarks used below.
Note that all benchmarks are sized to the limits of the current state-of-the-art by~\cite{BDGBG23SimplifyingStep-WiseExplanationSequences}.
Hence, while combinatorial solvers can solve much larger problems compared to the ones described below, techniques for generating step-wise explanations are currently limited to smaller instances.

\paragraph{\benchsudoku}
We use 100 9x9 Sudoku instances of ``expert'' difficulty generated by the QQwing tool~\cite{qqwing}.
These are modeled using \talldiff constraints, and they are made unsatisfiable by introducing a ``user mistake'' in one of the unfilled cells.
The introduced mistakes are nontrivial, meaning that each row, block, and column still has all distinct values in the given hints.

\paragraph{\benchjobshop}
In jobshop, the goal is to schedule $n$ jobs, each consisting of $m$ individual tasks, on $m$ different machines.
The tasks corresponding to the same job have to be processed in order, and hence, they share a precedence constraint in the model.
No machine can run two tasks at once.
We generate 100 jobshop instances, all with 5 jobs, consisting of 5 individual tasks.
The CP-models use the \tcumulative global constraint, which is usual in scheduling problems.

\paragraph{\benchdebug}
We collect 102 example models from different sources, such as the CSP-lib library~\cite{GW99CSPLIBBenchmarkLibraryConstraints}, the repository of examples of Hakan Kjellerstand\footnote{https://www.hakank.org/cpmpy/}, and the examples included in the \cpmpy repository.
Following the approach from \citet{LT17DebuggingUnsatisfiableConstraintModels}, we introduce a modeling error such as an off-by-one error in a \telement constraint, or swapping a comparison (e.g., swapping $\leq$ for $<$).
The modeling errors are introduced in such a way that each individual constraint in the CSP remains satisfiable.

\input{example}

\newpage
\section{Detailed results}
\Cref{tab:proofsize:jobshop,tab:proofsize:sudoku,tab:proofsize:debug} show the size of the proof after each stage in the pipeline.
\begin{itemize}
    \item ``Proof'' represents the proof as produced by the solver.
    \item ``No aux vars'' represents the proof after simplifying all nogoods involving auxiliary variables from the proof. Note that for both the \benchjobshop and \benchsudoku benchmarks, no auxiliary variables are present in the proofs. Indeed, the Pumpkin solver supports all (global) constraints exactly as modeled in these benchmarks. Thus, no reformulations are required by the modeling system.
    \item ``User cons'' represents the proof after replacing the solver-level constraints with user-level ones. Naturally, this does not alter the set of nogoods in the proof.
    \item ``Minimized 1'' represents the proof after the first minimization phase.
    \item ``Domain reductions'' represents the proof after simplifying any nogood that does not represent a domain reduction (i.e., has more than 1 variable).
    \item ``Minimized 2'' represents the proof after the second, optional minimization phase.
\end{itemize}

\begin{table*}[h]
\centering
\begin{tabular}{l|cccccc}
\toprule
 & Proof & No aux vars & User cons & Minimized 1 & Domain reductions & Minimized 2 \\
\midrule
Trim & 147.6 & 147.6 & 147.6 & 50.7 & 28.0 & 28.0 \\
Trim + Min. loc & 147.6 & 147.6 & 147.6 & 50.7 & 28.0 & 2.6 \\
Trim + Min. glob & 147.6 & 147.6 & 147.6 & 50.7 & 28.0 & 2.2 \\
Min. loc & 147.6 & 147.6 & 147.6 & 3.0 & 2.8 & 2.8 \\
Min. loc + Min. loc & 147.6 & 147.6 & 147.6 & 3.0 & 2.8 & 2.6 \\
Min. glob & 147.6 & 147.6 & 147.6 & 3.2 & 3.0 & 3.0 \\
Min. glob + Min. loc & 147.6 & 147.6 & 147.6 & 3.2 & 3.0 & 2.9 \\
\bottomrule
\end{tabular}
\caption{Size of proof in terms of nogoods after each step in the pipeline for \benchjobshop benchmark}
\label{tab:proofsize:jobshop}
\end{table*}

\begin{table*}[h]
\centering
\begin{tabular}{l|cccccc}
\toprule
 & Proof & No aux vars & User cons & Minimized 1 & Domain reductions & Minimized 2 \\
\midrule
Trim & 470.5 & 470.5 & 470.5 & 175.4 & 171.4 & 171.4 \\
Trim + Min. loc & 470.5 & 470.5 & 470.5 & 175.4 & 171.4 & 154.3 \\
Trim + Min. glob & 470.5 & 470.5 & 470.5 & 175.4 & 171.4 & 96.9 \\
Min. loc & 470.5 & 470.5 & 470.5 & 166.0 & 162.9 & 162.9 \\
Min. loc + Min. loc & 470.5 & 470.5 & 470.5 & 166.0 & 162.9 & 155.4 \\
Min. glob & 470.5 & 470.5 & 470.5 & 109.6 & 108.0 & 108.0 \\
Min. glob + Min. loc & 470.5 & 470.5 & 470.5 & 109.6 & 108.0 & 102.1 \\
\bottomrule
\end{tabular}
\caption{Size of proof in terms of nogoods after each step in the pipeline \benchsudoku benchmark}
\label{tab:proofsize:sudoku}
\end{table*}

\begin{table*}[h]
\centering
\begin{tabular}{l|cccccc}
\toprule
 & Proof & No aux vars & User cons & Minimized 1 & Domain reductions & Minimized 2 \\
\midrule
Trim & 229.7 & 77.9 & 77.9 & 37.0 & 17.3 & 17.3 \\
Trim + Min. loc & 229.7 & 77.9 & 77.9 & 37.0 & 17.3 & 13.1 \\
Trim + Min. glob & 229.7 & 77.9 & 77.9 & 37.0 & 17.3 & 12.8 \\
Min. loc & 229.7 & 77.9 & 77.9 & 17.0 & 14.0 & 14.0 \\
Min. loc + Min. loc & 229.7 & 77.9 & 77.9 & 17.0 & 14.0 & 13.1 \\
Min. glob & 229.7 & 77.9 & 77.9 & 18.4 & 14.7 & 14.7 \\
Min. glob + Min. loc & 229.7 & 77.9 & 77.9 & 18.4 & 14.7 & 13.8 \\
\bottomrule
\end{tabular}
\caption{Size of proof in terms of nogoods after each step in the pipeline for \benchdebug benchmark}
\label{tab:proofsize:debug}
\end{table*}

\fi

\end{document}

%% file: auxiliary.tex
\usepackage{xcolor}

\newcommand{\pageplan}[1]{{\small{\textit{\color{red!50!black} PP: #1}}}}
\renewcommand{\pageplan}[1]{}

\newcommand{\ignore}[1]{}

\newtheorem{sidenote}{Side note}

\usepackage{tcolorbox}

\newcommand{\variables}{\mathcal{X}}
\newcommand{\variablesV}{\mathcal{V}}
\newcommand{\domains}{\mathcal{D}}

\newcommand{\range}[2]{[#1,#2]}
\newcommand{\constraints}{\mathcal{C}}
\newcommand{\scope}[1]{\mathit{scope}(#1)}

\newcommand{\sols}[1]{\mathit{sols}(#1)}
\newcommand{\solsto}[2]{\mathit{sols}_{#2}(#1)}

\newcommand{\atomic}[1]{\ensuremath{\langle #1 \rangle}}

\newcommand{\constext}[1]{\textup{\mm{\textsc{#1}}}\xspace}
\newcommand{\consargs}[2]{\mm{\constext{#1}(#2)}\xspace}
\newcommand{\alldiff}[1]{\consargs{Distinct}{#1}}
\newcommand{\talldiff}{\constext{Distinct}}

\newcommand{\tcumulative}{\constext{Cumulative}}

\newcommand{\telement}{\constext{Element}}

\newcommand{\cpmpy}{CPMpy\xspace}
\newcommand{\minizinc}{MiniZinc\xspace}

\renewcommand{\implies}{\rightarrow}

\usepackage{ amssymb }

\newcommand\muscall[2]{\call{MUS}(\mathit{soft}\!:\! #1, \mathit{hard}\!:\!#2)}

\usepackage{amsmath}
\usepackage{xspace}
\newcommand\mm[1]{\ensuremath{#1}\xspace}
\newcommand\call[1]{\mm{\textsc{#1}}}
\newcommand{\integers}{\mathbb{Z}}

\newcommand{\set}[1]{\ensuremath{\{#1\}}}

\renewcommand{\implies}{\Rightarrow}

\usepackage{wasysym}

\newcommand{\clause}{\mm{\gamma}}

\usepackage{bm}

\newcommand{\pformat}{DRCP\xspace}

\newcommand{\cons}{\ensuremath{C}}
\newcommand{\reason}{\ensuremath{R}}
\newcommand{\req}{\mathit{Req}}

\newcommand{\step}[1]{\stepp{\cons_{#1}}{\reason_{#1}}}
\newcommand{\abstractproof}[1]{[\step{1}, \step{2}, \dots, \step{#1}]}
\newcommand{\mathabstractproof}[1]{$\abstractproof{#1}$\xspace}
\newcommand{\stepp}[2]{(#2, #1)} 
\newcommand{\cell}{\mathit{cell}}
\renewcommand{\cell}{x}

\newcommand{\benchformat}[1]{\textbf{#1}\xspace}
\newcommand{\benchsudoku}{\benchformat{Sudoku}}
\newcommand{\benchjobshop}{\benchformat{Job-shop}}
\newcommand{\benchdebug}{\benchformat{Modeling examples}}

\newcommand{\approachformat}[1]{\textbf{#1}\xspace}

\newcommand{\approachtrim}{\approachformat{Trim}}
\newcommand{\approachloc}{\approachformat{Min. Loc}}
\newcommand{\approachglob}{\approachformat{Min. Glob}}

\newcommand{\approachtrimloc}{\approachformat{Trim + Min. loc}}
\newcommand{\approachtrimglob}{\approachformat{Trim + Min. glob}}

\newcommand{\approachlocloc}{\approachformat{Min. loc + Min. loc}}
\newcommand{\approachglobloc}{\approachformat{Min. glob + Min. loc}}

\newcommand{\sota}{\approachformat{\citeauthor{BDGBG23SimplifyingStep-WiseExplanationSequences}~(\citeyear{BDGBG23SimplifyingStep-WiseExplanationSequences})}}
\renewcommand{\sota}{\approachformat{\citeauthor{BDGBG23SimplifyingStep-WiseExplanationSequences}}} 
\renewcommand{\sota}{\approachformat{SimplifyGreedy}}

\newcommand{\softcons}{\constraints_s}
\newcommand{\hardcons}{\constraints_h}

\usepackage{enumitem}
\newlist{EQ}{enumerate}{1}
\setlist[EQ]{label=\bfseries{}EQ\arabic*.,leftmargin=30pt}


%% file: bb_common-citations.tex
\usepackage{etoolbox}

\ExplSyntaxOn

\newcommand\setcitation[2]{%
	\csdef{mycommoncitation\text_uppercase:n{#1}}{#2}%
	\csappto{bbAllCommonCitations}{\cite{#2}\ }
}
\newcommand\getcitation[1]{%
	\csuse{mycommoncitation\text_uppercase:n{#1}}}

\newcommand\refto[1]{%
	\ifcsname  mycommoncitation\text_uppercase:n{#1}\endcsname%
	\getcitation{#1}%
	\else%
	#1%
	\fi%
}

\ExplSyntaxOff

\newcommand\mycite[1]{\cite{\refto{#1}}}

\setcitation{IDP}{DBBJD18PredicatelogicmodelinglanguageIDPsystem}
\setcitation{fodot}{DT08logicnonmonotoneinductivedefinitions}
\setcitation{CPsupport}{DBDD13ModelExpansionPresenceFunctionSymbolsUsing}
\setcitation{minisatid}{DBDD13ModelExpansionPresenceFunctionSymbolsUsing}
\setcitation{FunctionDetection}{DB13Detectionexploitationfunctionaldependenciesmodelgeneration}
\setcitation{lazyGrounding}{DDBS15LazyModelExpansionInterleavingGroundingSearch} 
\setcitation{Bootstrapping}{BJDJBD16BootstrappingInferenceIDPKnowledgeBaseSystem}
\setcitation{BootstrappingIDP}{BJDJBD16BootstrappingInferenceIDPKnowledgeBaseSystem}
\setcitation{FOSymmetry}{DBBD16localdomainsymmetrymodelexpansion}
\setcitation{FOLASP}{VDV21FOLASPFOInputLanguageAnswerSet}
\setcitation{inputster}{JJJ13CompilingInputFOinductivedefinitionsinto}
\setcitation{LearningPaper}{BBBDDJLR15Predicatelogicmodelinglanguagemodelingsolving}

\setcitation{foid}{DT08logicnonmonotoneinductivedefinitions}
\setcitation{cplogic}{VDB10FOIDextensionDLrules}
\setcitation{clog}{BVDV14FOCKnowledgeRepresentationLanguageCausality} 
\setcitation{foc}{BVDV14FOCKnowledgeRepresentationLanguageCausality}
\setcitation{inferenceClog}{BVDV14InferenceFOCModellingLanguage}
\setcitation{inferenceFOC}{BVDV14InferenceFOCModellingLanguage}
\setcitation{examplesClog}{BVDV14FOCRelatedModellingParadigms} 
\setcitation{satid}{MWDB08SATIDSatisfiabilityPropositionalLogicExtendedInductive}

\setcitation{fodot2asp}{DLTV19informalsemanticsAnswerSetProgrammingTarskian} 
\setcitation{Tarskian}{DLTV19informalsemanticsAnswerSetProgrammingTarskian} 
\setcitation{TarskianSemanticsASP}{DLTV19informalsemanticsAnswerSetProgrammingTarskian} 
\setcitation{KBS}{DV08BuildingKnowledgeBaseSystemIntegrationLogic}
\setcitation{KBS-invitedtalk}{D16FOKnowledgeBaseSystemproject}
\setcitation{KBPE}{DWD11PrototypeKnowledge-BasedProgrammingEnvironment}

\setcitation{justifications}{DBS15FormalTheoryJustifications} 
\setcitation{justificationTheory}{DBS15FormalTheoryJustifications} 
\setcitation{AFT}{DMT00ApproximationsStableOperatorsWell-FoundedFixpointsApplications}

\setcitation{SAT}{BHvW21HandbookSatisfiability-SecondEdition}
\setcitation{HandbookOfSAT}{BHvW21HandbookSatisfiability-SecondEdition}
\setcitation{SMS}{KS24SATModuloSymmetriesGraphGenerationEnumeration}

\setcitation{totalizer}{BB03EfficientCNFEncodingBooleanCardinalityConstraints}

\setcitation{satsuma}{ABR24satsumaStructure-BasedSymmetryBreakingSAT}
\setcitation{BreakID}{DBBD16ImprovedStaticSymmetryBreakingSAT}
\setcitation{symchaff}{S09SymChaffexploitingsymmetrystructure-awaresatisfiabilitysolver}

\setcitation{saucy}{KSM10SymmetrySatisfiabilityUpdate}
\setcitation{dejavu}{AS21EngineeringFastProbabilisticIsomorphismTest}
\setcitation{nauty}{MP14PracticalgraphisomorphismII}
\setcitation{bliss}{JK07EngineeringEfficientCanonicalLabelingToolLarge}

\setcitation{DPLLT}{GHNOT04DPLLTFastDecisionProcedures}
\setcitation{SMT}{BSST21SatisfiabilityModuloTheories}

\setcitation{CP}{RvW06HandbookConstraintProgramming}
\setcitation{csp2asp}{DW11Translation-BasedConstraintAnswerSetSolving}
\setcitation{EZCSP}{B11Industrial-SizeSchedulingASPCP}
\setcitation{Inca}{DW12AnswerSetSolvingLazyNogoodGeneration}
\setcitation{Zinc}{MNRSGW08DesignZincModellingLanguage}
\setcitation{MiniZinc}{NSBBDT07MiniZincTowardsStandardCPModellingLanguage}

\setcitation{ASPComp2}{DVBGT09SecondAnswerSetProgrammingCompetition}
\setcitation{ASPComp3}{CIR14thirdopenanswersetprogrammingcompetition}
\setcitation{ASPComp4}{ACCDDIKK13FourthAnswerSetProgrammingCompetitionPreliminary}
\setcitation{ASPComp5}{CGMR16DesignresultsFifthAnswerSetProgramming}
\setcitation{ASPComp6}{GMR17SixthAnswerSetProgrammingCompetition}
\setcitation{ASPComp7}{GMR20SeventhAnswerSetProgrammingCompetitionDesign}
\setcitation{AspInPractice}{GKKS12AnswerSetSolvingPractice}
\setcitation{clasp}{GKS12Conflict-drivenanswersetsolvingFromtheory}
\setcitation{oclingo}{GGKOSS12StreamReasoningAnswerSetProgrammingPreliminary}
\setcitation{clingo}{GKKOSW16TheorySolvingMadeEasyClingo5}
\setcitation{gringo}{GST07GrinGoNewGrounderAnswerSetProgramming}
\setcitation{cmodels}{GLM04SAT-BasedAnswerSetProgramming}
\setcitation{DLV}{LPFEGPS06DLVsystemknowledgerepresentationreasoning}
\setcitation{GroundingBottleneck}{SPL14TwoApplicationsASP-PrologSystemDecomposablePrograms}
\setcitation{lazygroundingASP}{BW18ExploitingJustificationsLazyGroundingAnswerSet} 
\setcitation{justificationsAlpha}{BW18ExploitingJustificationsLazyGroundingAnswerSet}
\setcitation{ASP}{MT99StableModelsAlternativeLogicProgrammingParadigm}
\setcitation{OLL}{AKMS12Unsatisfiability-basedoptimizationclasp} 

\setcitation{Hilog}{CKW93HILOGFoundationHigher-OrderLogicProgramming}

\setcitation{MaxSAT}{LM21MaxSATHardSoftConstraints}

\setcitation{KR}{vLP08HandbookKnowledgeRepresentation}

\setcitation{lazyclausegeneration}{OSC09Propagationlazyclausegeneration}
\setcitation{FP}{H89ConceptionEvolutionApplicationFunctionalProgrammingLanguages}
\setcitation{GroundingWithBounds}{WMD10GroundingFOFOIDBounds}
\setcitation{GroundWithBounds}{WMD10GroundingFOFOIDBounds}
\setcitation{LTC}{BJBDVD14SimulatingDynamicSystemsUsingLinearTime}
\setcitation{SPSAT}{DBDDM12SymmetryPropagationImprovedDynamicSymmetryBreaking}
\setcitation{LCG}{S10LazyClauseGenerationCombiningPowerSAT}

\setcitation{GroundedFixpoints}{BVD15Groundedfixpointstheirapplicationsknowledgerepresentation}
\setcitation{PartialGroundedFixpoints}{BVD15PartialGroundedFixpoints}
\setcitation{LogicBlox}{GAK12LogicBloxPlatformLanguageTutorial}
\setcitation{proB}{LB08ProBautomatedanalysistoolsetBmethod}
\setcitation{NaturalInductions}{DV14Well-FoundedSemanticsPrincipleInductiveDefinitionRevisited} 
\setcitation{LP}{vK76SemanticsPredicateLogicProgrammingLanguage}

\setcitation{AF}{D95AcceptabilityArgumentsFundamentalRoleNonmonotonicReasoning}
\setcitation{ADF}{BW10AbstractDialecticalFrameworks}
\setcitation{ADFRevisited}{BSEWW13AbstractDialecticalFrameworksRevisited}
\setcitation{DefaultLogic}{R80LogicDefaultReasoning}
\setcitation{DL}{R80LogicDefaultReasoning}
\setcitation{AEL}{M85SemanticalConsiderationsNonmonotonicLogic}
\setcitation{minisat}{ES03ExtensibleSAT-solver}
\setcitation{completion}{C77NegationFailure}
\setcitation{ClarkCompletion}{C77NegationFailure}
\setcitation{wasp}{ADFLR13WASPNativeASPSolverBasedConstraint}
\setcitation{CEGAR}{CGJLV03Counterexample-guidedabstractionrefinementsymbolicmodelchecking}
\setcitation{CuttingPlane}{DFJ54SolutionLarge-ScaleTraveling-SalesmanProblem}
\setcitation{CuttingPlanes}{DFJ54SolutionLarge-ScaleTraveling-SalesmanProblem}
\setcitation{kodkod}{TJ07KodkodRelationalModelFinder}
\setcitation{CDCL}{SS99GRASPSearchAlgorithmPropositionalSatisfiability}
\setcitation{1UIP}{ZMMM01EfficientConflictDrivenLearningBooleanSatisfiability}
\setcitation{relevance}{JBDJD16RelevanceSATID}
\setcitation{relevance-implementation}{JDBJD16ImplementingRelevanceTrackerModule}
\setcitation{WFS}{VRS91Well-FoundedSemanticsGeneralLogicPrograms}
\setcitation{UnfoundedSet}{VRS91Well-FoundedSemanticsGeneralLogicPrograms}
\setcitation{UFS}{VRS91Well-FoundedSemanticsGeneralLogicPrograms}
\setcitation{stablesemantics}{GL88StableModelSemanticsLogicProgramming}
\setcitation{shatter}{ASM06EfficientSymmetryBreakingBooleanSatisfiability}
\setcitation{sbass}{DTW11Symmetry-breakinganswersetsolving}
\setcitation{lparsemanual}{url:lparsemanual}
\setcitation{AIC}{FGZ04Activeintegrityconstraints}
\setcitation{templates}{DvJD15Semanticstemplatescompositionalframeworkbuildinglogics}
\setcitation{templates2}{DvBJD16CompositionalTypedHigher-OrderLogicDefinitions}
\setcitation{sat-to-sat}{JTT16SAT-to-SATDeclarativeExtensionSATSolversNew}
\setcitation{sat-to-sat-qbf}{BJT16SolvingQBFInstancesNestedSATSolvers}
\setcitation{sat-to-sat-SO}{BJT16DeclarativeSolverDevelopmentCaseStudies}
\setcitation{XSB}{SW12XSBExtendingPrologTabledLogicProgramming}
\setcitation{KCmap}{DM02KnowledgeCompilationMap}
\setcitation{KnowledgeCompilationMap}{DM02KnowledgeCompilationMap}
\setcitation{TLA}{L02SpecifyingSystemsTLALanguageToolsHardware}
\setcitation{EventB}{A10ModelingEvent-B-SystemSoftwareEngineering}
\setcitation{MX}{MT05FrameworkRepresentingSolvingNPSearchProblems}
\setcitation{MIP}{S96Linearintegerprogramming-theorypractice}
\setcitation{PerfectModel}{P88DeclarativeSemanticsDeductiveDatabasesLogicPrograms}
\setcitation{SafeInductions}{BVD17SafeInductionsAlgebraicStudy}
\setcitation{alpha}{W17BlendingLazy-GroundingCDNLSearchAnswer-SetSolving}
\setcitation{omiga}{DEFWW12OMiGAOpenMindedGroundingOn-The-FlyAnswer}
\setcitation{gasp}{PDPR09GASPAnswerSetProgrammingLazyGrounding}
\setcitation{asperix}{LN09FirstVersionNewASPSolverASPeRiX}
\setcitation{CTL}{CE81DesignSynthesisSynchronizationSkeletonsUsingBranching-Time}
\setcitation{AFT-AIC}{BC18Fixpointsemanticsactiveintegrityconstraints}
\setcitation{UltimateApproximator}{DMT04Ultimateapproximationapplicationnonmonotonicknowledgerepresentation}
\setcitation{KripkeKleene}{F85Kripke-KleeneSemanticsLogicPrograms}
\setcitation{AFT-HO}{CRS18ApproximationFixpointTheoryWell-FoundedSemanticsHigher-Order} 
\setcitation{HereThere}{H30formalenRegelnintuitionistischenLogik}
\setcitation{dAEL}{VCBD16DistributedAutoepistemicLogicApplicationAccessControl}
\setcitation{SDD}{D11SDDNewCanonicalRepresentationPropositionalKnowledge}
\setcitation{HEX}{EIST06dlvhexProverSemantic-WebReasoningunderAnswer-Set}
\setcitation{wADF}{BSWW18WeightedAbstractDialecticalFrameworks}
\setcitation{wADFfix}{BPSWW18WeightedAbstractDialecticalFrameworksExtendedRevised}
\setcitation{TransitionSystems}{NOT06SolvingSATSATModuloTheoriesFrom}
\setcitation{galliwasp}{MG12GalliwaspGoal-DirectedAnswerSetSolver}
\setcitation{clingcon}{BKOS17Clingconnextgeneration}
\setcitation{lp2sat}{JN11CompactTranslationsNon-disjunctiveAnswerSetPrograms}
\setcitation{lp2mip}{LJN12AnswerSetProgrammingMixedIntegerProgramming}
\setcitation{lp2diff}{JNS09ComputingStableModelsReductionsDifferenceLogic}
\setcitation{lp2acyc}{GJR14AnswerSetProgrammingSATmoduloAcyclicity}
\setcitation{PB}{RM09Pseudo-BooleanCardinalityConstraints}
\setcitation{CuttingPlanes}{CCT87complexitycutting-planeproofs}
\setcitation{RoundingSAT}{EN18DivideConquerTowardsFasterPseudo-BooleanSolving}
\setcitation{PRS}{DG02InferenceMethodsPseudo-BooleanSatisfiabilitySolver}
\setcitation{sat4j}{LP10Sat4jlibraryrelease22}
\setcitation{mingo}{LJN12AnswerSetProgrammingMixedIntegerProgramming}
\setcitation{pbmodels}{LT05Pbmodels-SoftwareComputeStableModels}
\setcitation{HEF-LP}{GLL07Head-Elementary-Set-FreeLogicPrograms}
\setcitation{HCF-LP}{BD94PropositionalSemanticsDisjunctiveLogicPrograms}
\setcitation{aspcore2}{CFGIKKLM20ASP-Core-2InputLanguageFormat}
\setcitation{SemanticWeb}{AGvH12SemanticWebPrimer3rdEdition}
\setcitation{QBF}{KB21TheoryQuantifiedBooleanFormulas}
\setcitation{SMUS}{IPLM15SmallestMUSExtractionMinimalHittingSet}
\setcitation{CDCLsym}{MBCK18CDCLSymIntroducingEffectiveSymmetryBreakingSAT}
\setcitation{SEL}{DBB17SymmetricExplanationLearningEffectiveDynamicSymmetry}
\setcitation{Coq}{BC04InteractiveTheoremProvingProgramDevelopment-}
\setcitation{Isabelle}{NPW02IsabelleHOL-ProofAssistantHigher-OrderLogic}
\setcitation{HOL}{SN08BriefOverviewHOL4}
\setcitation{lean}{dU21Lean4TheoremProverProgrammingLanguage}
\setcitation{CakeML}{TMKFON19verifiedCakeMLcompilerbackend}
\setcitation{FRAT}{BCH22FlexibleProofFormatSATSolver-ElaboratorCommunication}
\setcitation{DRUP}{GN03VerificationProofsUnsatisfiabilityCNFFormulas}
\setcitation{RUP}{GN03VerificationProofsUnsatisfiabilityCNFFormulas}
\setcitation{GRIT}{CMS17EfficientCertifiedResolutionProofChecking}
\setcitation{TraceCheck}{url:TraceCheck}
\setcitation{LRAT}{CHHKS17EfficientCertifiedRATVerification}

\setcitation{PR}{HKB17ShortProofsWithoutNewVariables}
\setcitation{DRAT}{WHH14DRAT-trimEfficientCheckingTrimmingUsingExpressive}
\setcitation{satcomp2016}{BHJ17SATCompetition2016RecentDevelopments}
\setcitation{satcomp2013}{url:SATcomp2013}

\setcitation{Essence}{FHJHM08Essenceconstraintlanguagespecifyingcombinatorialproblems}
\setcitation{smodels}{SN01SmodelsSystem}
\setcitation{lparse}{S98ImplementationLocalGroundingLogicProgramsStable}
\setcitation{IDPZ3}{CVVD22IDP-Z3reasoningengineFO} 
\setcitation{IDP-Z3}{CVVD22IDP-Z3reasoningengineFO} 
\setcitation{Z3}{dB08Z3EfficientSMTSolver}

\setcitation{DMN}{DMN} 
\setcitation{Planning}{GNT04Automatedplanning-theorypractice}
\setcitation{AutomatedPlanning}{GNT04Automatedplanning-theorypractice}
\setcitation{RC2}{IMM19RC2EfficientMaxSATSolver}
\setcitation{enfragmo}{phd/Aavani14}
\setcitation{TPTP}{url:tptp}

%% file: pipeline_aaai.tex
\centering
\begin{tikzpicture}[node distance=0.1cm and 0.5cm,
    startstop/.style={rectangle, rounded corners, minimum height=1cm, text centered, draw=black, fill=blue!40},
    process/.style={rectangle, rounded corners, minimum height=1cm, text centered, draw=black, fill=gray!40},
    processyellow/.style={rectangle, rounded corners, minimum height=1cm, minimum width=2.5cm, text centered, draw=black, fill=yellow!40},
    trim/.style={rectangle, rounded corners, minimum height=0.7cm, text centered, draw=black, fill=pink!40},
    squash/.style={rectangle, rounded corners, minimum height=1cm, text centered, draw=black, fill=green!20},
    decision/.style={rectangle, rounded corners, minimum height=0.7cm, minimum width=2.2cm, text centered, draw=black, fill=red!20},
    toseq/.style={rectangle, rounded corners, minimum height=1cm, text centered, draw=black, fill=brown!40},
    arrow/.style={thick,->,>=stealth}
]

\node (solve) [startstop, text width=1.5cm] {Solve + Parse proof};
\node (auxvars) [squash, right=of solve, text width=2cm] {Simplify to user-variables};
\node (solvercons) [process, right=of auxvars, text width=2cm] {User-level constraints};
\node (trimnode) [trim, right=of solvercons] {Trim};
\node (squashnode) [squash, right=of trimnode, xshift=1cm, text width=2cm] {Simplify to domain reductions};

\node (minlocal) [trim, above left=of squashnode, xshift=1cm] {Trim with local minimization};
\node (minglobal) [trim, below left=of squashnode, xshift=1cm] {Trim with global minimization};
\node (minlocall) [decision, above right=of squashnode,xshift=-1.2cm] {Trim with local minimization};
\node (mingloball) [decision, below right=of squashnode,xshift=-1.2cm] {Trim with global minimization};
\node (toseq) [toseq, right=of squashnode,text width=1.8cm,xshift=2.5cm] {Merge steps and print explanation};

\draw [arrow, line width=1.5] (solve) -- (auxvars);
\draw [arrow, line width=1.5] (auxvars) -- (solvercons);
\draw [arrow] (solvercons) -- (minglobal);
\draw [arrow] (solvercons) -- (minlocal);
\draw [arrow, line width=1.5] (solvercons) -- (trimnode);
\draw [arrow, line width=1.5] (trimnode) -- (squashnode);
\draw [arrow] (minlocal) -- (squashnode);
\draw [arrow] (minglobal) -- (squashnode);
\draw [arrow, line width=1.5] (squashnode) -- (toseq);
\draw [arrow] (squashnode) -- (minlocall);
\draw [arrow] (squashnode) -- (mingloball);
\draw [arrow] (minlocall) -- (toseq);
\draw [arrow] (mingloball) -- (toseq);

\end{tikzpicture}

%% file: results_table.tex
\setlength{\tabcolsep}{1mm}


\small
\begin{tabular}{l|rrrr|rrrr|rrrr}
Benchmark  & \multicolumn{4}{c|}{\textbf{Job-shop (100inst})} & \multicolumn{4}{c|}{\textbf{Sudoku (100 inst)}} & \multicolumn{4}{c}{\textbf{Modeling examples (73 inst)}} \\\midrule
 & \multicolumn{2}{c}{Sequence length} & \multicolumn{2}{c|}{Max stepsize} & \multicolumn{2}{c}{Sequence length} & \multicolumn{2}{c|}{Max stepsize} & \multicolumn{2}{c}{Sequence length} & \multicolumn{2}{c}{Max stepsize} \\
 & Avg (±std) & Med & Avg (±std) & Med & Avg (±std) & Med & Avg (±std) & Med & Avg (±std) & Med & Avg (±std) & Med \\
\midrule
Trim & 24.3 (±7.9) & 25.0 & 2.7 (±2.5) & 1.0 & 85.9 (±37.4) & 90.0 & 7.5 (±7.8) & 4.0 & 11.5 (±17.6) & 5.0 & 3.5 (±5.8) & 1.0 \\
Trim + Min. loc & 2.5 (±4.1) & 1.0 & 1.0 (±0.0) & 1.0 & 82.0 (±36.1) & 83.5 & 4.2 (±4.4) & 2.0 & 8.6 (±15.8) & 3.0 & 2.1 (±2.8) & 1.0 \\
Trim + Min. glob & 2.2 (±2.9) & 1.0 & 1.0 (±0.0) & 1.0 & 62.6 (±35.9) & 57.5 & 3.4 (±3.6) & 1.0 & 8.6 (±16.0) & 3.0 & 2.0 (±2.6) & 1.0 \\
SimplifyGreedy & 2.1 (±2.8) & 1.0 & 1.0 (±0.0) & 1.0 & 37.2 (±22.0) & 33.0 & 1.2 (±0.8) & 1.0 & 7.3 (±13.9) & 2.0 & 1.4 (±0.7) & 1.0 \\
\end{tabular}
\normalsize

\addtolength{\tabcolsep}{0.2em}

%% file: example.tex
\section{Example from proof to stepwise explanation}

Consider a simplified jobshop problem with 4 tasks $a,b,c,d$.
Jobs $a$ and $c$ should be ran on machine 1, and jobs $b$ and $d$ on machine 2.
Job $a$ has a duration of 3, job $b$ a duration of 4, job $c$ a duration of 4 and job $d$ a duration of 5.
Additionally, job $a$ should be finished before job $b$ can start, and similarly, job $c$ should finish before the start of job $d$.
We model this problem using 4 integer variables $a,b,c,d$ to represent the start of each task, and the following constraints:
\begin{align*}
    \constraints = \{
         &\consargs{NoOverlap}{(a,3),(c,4)} \\
         &\consargs{NoOverlap}{(b,4),(d,5)} \\
          &a + 3 \leq b \\ 
          &c + 4 \leq d\}
 \end{align*}

In this example, we assume the solver does not support the \constext{NoOverlap} constraint, and thus it must be decomposed.
We use the following decomposition, with $x$ a freshly introduced Boolean auxiliary variable:
\begin{align*}
    \consargs{NoOverlap}{(a,3),(c,4)} \equiv 
        \{\quad&x \implies a + 3 \leq c, \\
          \neg &x \implies c + 4 \leq a\quad\}
\end{align*}

For domains $a,b,c,d \in \range{0}{6}$, the solver produces the following proof 

\newcommand{\proofitem}[2]{\item 
                            \begin{itemize}
                                \item $\reason = #1$
                                \item $\cons = #2$
                            \end{itemize}}

\newcounter{proofcounter}
\setcounter{proofcounter}{0}
\renewcommand{\proofitem}[2]{\stepcounter{proofcounter}
                             \arabic{proofcounter}&
                                $#1$ & $#2$ \\}

\setcounter{proofcounter}{0}
\begin{table}[h]
    \centering
\small
\begin{tabular}{l|l|l}
    Id & Reason & Constraint \\ \midrule
\proofitem{a + 3 \leq b}{(a \leq 3) \vee (b \geq 7)}
\proofitem{1}{a \leq 3}
\proofitem{a + 3 \leq b}{(a \leq -1) \vee (b \geq 3)}
\proofitem{4}{b \geq 3}
\proofitem{\neg x_1 \implies c + 4 \leq a}{(x_1 \geq 1) \vee (a \geq 4) \vee (c \leq -1)}
\proofitem{2,5}{x_1 \geq 1}
\proofitem{x_2 \implies b + 4 \leq d}{(x_2 \leq 0) \vee (b \leq 2) \vee (d \geq 7)}
\proofitem{4,7}{x_2 \leq 0}
\proofitem{x_1 \implies a + 3 \leq c}{(x_1 \leq 0) \vee (a \leq -1) \vee (c \geq 3)}
\proofitem{6,9}{c \geq 3}
\proofitem{\neg x_2 \implies d + 5 \leq d}{(x_2 \geq 1) \vee (b \geq 7) \vee (d \leq 1)}
\proofitem{8,11}{d \leq 1}
\proofitem{c + 4 \leq d}{(c \leq 2) \vee (d \geq 2)}
\proofitem{10,12,13}{\bot}
\end{tabular}
\normalsize
\caption{DRCP-proof for jobshop example}
\end{table}

We then convert the proof to a step-wise explanation sequence by replacing the solver-level constraints with user-level ones, and by simplifying (i.e., removing) steps that derive auxiliary variables, or clauses involving more than one variable.
The resulting proof is the following.

\setcounter{proofcounter}{0}
\begin{table}[h]
    \centering
\small
\begin{tabular}{l|l|l}
    Id & Reason & Constraint \\ \midrule
\proofitem{a + 3 \leq b}{a \leq 3}
\proofitem{a + 3 \leq b}{b \geq 3}
\proofitem{1, \consargs{NoOverlap}{(a,3),(c,4)}}{c \geq 3}
\proofitem{2, \consargs{NoOverlap}{(b,4),(d,5)}}{d \leq 1}
\proofitem{3,4, c + 4 \leq d}{\bot}
\end{tabular}
\normalsize
\caption{Step-wise explanation for jobshop example}
\end{table}

By applying our pipeline to this proof (\approachtrimglob in this case), we obtain the following proof, which is equivalent to an explanation sequence.
Indeed, each step in the proof derives a simple ``fact'' about the domain of variable, and uses a combination of previously derived facts, and user constraints.

\setcounter{proofcounter}{0}
\begin{table}[h]
    \centering
    \begin{tabular}{l|l|l}
    Id & Reason & Constraint \\ \midrule
        \proofitem{a + 3 \leq b}{a \leq 3}
        \proofitem{1, \consargs{NoOverlap}{(a,3),(c,4)}}{c \geq 3}
        \proofitem{3, c + 4 \leq d}{\bot}
    \end{tabular}
    \caption{Minimized explanation for jobshop example}
    \label{tab:my_label}
\end{table}